\documentclass[lettersize,journal]{IEEEtran}
\usepackage{amsmath,amsfonts}
\usepackage{algorithmic}
\usepackage{algorithm}
\usepackage{array}
\usepackage[caption=false,font=normalsize,labelfont=sf,textfont=sf]{subfig}
\usepackage{textcomp}
\usepackage{stfloats}
\usepackage{url}
\usepackage{verbatim}
\usepackage{graphicx}
\usepackage{cite}
\usepackage{amsmath, amssymb, amsthm}
\newtheorem{theorem}{Theorem}

\newtheorem{lemma}[theorem]{Lemma}

\newcommand{\jac}[1]{D\mkern-0.75mu{#1}}

\begin{document}

\title{Strong anti-Hebbian plasticity alters the convexity of network attractor landscapes}

\author{Lulu Gong$^{*}$,~\IEEEmembership{Member,~IEEE,}, Xudong Chen,~\IEEEmembership{Member,~IEEE,}, and ShiNung Ching,~\IEEEmembership{Member,~IEEE,},
\thanks{This work is partially supported by ARO-MURI  grant W911NF2110312 from the US Department of Defense. All authors are with the Department of Electrical and Systems Engineering, Washington University in St. Louis, 63130, USA.}
\thanks{The corresponding author: L. Gong;  Email address:  glulu@wustl.edu. }
}



\maketitle

\begin{abstract}
In this paper, we study recurrent neural networks in the presence of pairwise learning rules. We are specifically interested in how the attractor landscapes of such networks become altered as a function of the strength and nature (Hebbian vs. anti-Hebbian) of learning, which may have a bearing on the ability of such rules to mediate large-scale optimization problems. Through formal analysis, we show that a transition from Hebbian to anti-Hebbian learning brings about a pitchfork bifurcation that destroys convexity in the network attractor landscape. In larger-scale settings, this implies that anti-Hebbian plasticity will bring about multiple stable equilibria, and such effects may be outsized at interconnection or `choke' points. Furthermore, attractor landscapes are more sensitive to slower learning rates than faster ones. These results provide insight into the types of objective functions that can be encoded via different pairwise plasticity rules. 
\end{abstract}


\section{Introduction}
\IEEEPARstart{I}{n} the context of optimization and learning, pairwise rules present an enigma. On the one hand, they are well validated from a biological perspective and are thought to be the prevailing process of learning at the neuronal scale. On the other hand, it has proven to be a considerable bottleneck to generate functional instances of such rules in artificial neural network settings, and this is a persistent enigma in theoretical neuroscience. The attractor landscape, or vector field, of recurrent networks is a fundamental characterization of their time evolution, and hence can provide key information about how such networks can maximize objective functions. For instance, solving a convex objective would nominally require a network with a single, asymptotically stable equilibrium point and such is, of course, the premise underlying classical gradient-based optimization frameworks. The specific question we address here involves characterizing the sensitivity of attractor landscapes to variations in learning rate, or `strength' of plasticity. Implied in this is a differentiation of landscapes under Hebbian vs. anti-Hebbian learning. 

More broadly, the attractor landscape of recurrent networks is a frequent topic of theoretical study in other cognitive contexts \cite{deco2013brain,khona2022attractor}. Particularly in associative memory, asymptotically stable equilibria (i.e., fixed point attractors) in the state space of networks are thought to encode memories \cite{hopfield1982neural},  and the density of these equilibria is regarded as a measure of network memory capacity \cite{mceliece1987capacity,wu2012storage}. Understanding how memory capacity is achieved or maximized in finite-size networks has been a frequent topic of research in theoretical neuroscience, e.g., \cite{storkey1997increasing,krotov2016dense,demircigil2017model,krotov2020large}. Many of these analyses focus on how the synaptic connectivity, i.e., the weight matrix, of a network impacts attractor density and stability. However, in these works, the weight matrix is usually treated as being static. An important distinction in this regard is the goal of learning a set of weights that give rise to specific, multiple equilibria in a state-space of neural activity vs. ongoing plasticity wherein the weight configurations are themselves part of the state-space and hence attractor landscape.

Indeed, the attractor landscapes of recurrent neural networks (RNNs) with ongoing pairwise plasticity are less well-characterized, even in low-dimensional settings.
In \cite{dong1992dynamic,Zucker2012},  recurrent neural network models with standard Hebbian learning have been studied.
It is shown that under assumptions of symmetry in the dynamics, trajectories will evolve to stable equilibria in the state space.
In \cite{DAUCE1998521,Siri2007,Siri2008}, the authors numerically demonstrated that in large random RNNs, Hebbian-like learning can lead the system transition from chaos to having stable equilibria via bifurcations with respect to exogenous stimuli. 
Recently, similar models have also been analyzed using contraction theory \cite{kozachkov2020achieving,9993009}. There, it has been shown that the network dynamics are contracting under certain parameter conditions, which implies that the population activity of the network will converge towards well-defined limit sets. Our work is in a similar spirit, insofar as we examine how different forms of pairwise plasticity may alter the limiting behavior of the networks in question.

Motivated by these questions, in this study, we investigate the intrinsic dynamics of a general RNN with both Hebbian and anti-Hebbian learning (RNN-HL) forms within both uni- and bidirectional connectivity motifs. 
These minimal network structures serve as the building blocks of more complex RNN-HL networks, and thus it is useful to study their dynamic properties in depth. 
The nonlinearity inherent in the RNN-HL model, specifically through its activation functions, means that the exact analytical calculation of the attractor landscape is intractable in the general case. However, by making strategic simplifications, we are able to provide insights into how learning rate and direction lead to both standard and imperfect pitchfork bifurcations \cite{rajapakse2017pitchfork}. While such dynamics 
are common in neural dynamical models
\cite{hoppensteadt2012weakly,dayan2005theoretical,cheng2006multistability} with fixed connectivity, they have not been as well-studied in the presence of Hebbian learning dynamics.

Our results indicate the differential roles of anti-Hebbian vs. Hebbian mechanisms may have in shaping the attractor landscape of networks in the context of learning. Here, anti-Hebbian mechanisms are found to be associated with multiple stable equilibria that would preclude `global' convergent solutions. In contrast, Hebbian mechanisms encourage unique equilibrium and are more robust to parameter changes.
The remaining part is structured as follows. In Section II we
introduce mathematical models and the problem formulation. In Section III,  we first establish conditions for the unique and multiple equilibria for the general minimal models. Then, for the simplified system with bidirectional connections, we prove
the existence of a pitchfork bifurcation that results in the onset of multiple equilibria,  and we characterize the dynamics locally and globally under certain parameter conditions.  
In Section IV, We empirically confirm that the increase in the number of equilibria, resulting from variations in the learning rate parameter, is consistent in moderate-sized networks.
We conclude in Section V.

\section{Mathematical models of RNN with Hebbian learning}
\subsection{Simultaneous dynamics of RNN with generic Hebbian learning}
Consider a network of $n$ neurons. Let $x_i$ be the state of the neural node $i$ and $w_{ij}$ be the weight of the synapse from node $i$ to $j$. If there are no external inputs, according to \cite{kozachkov2020achieving}, the RNN with dynamic synaptic connections governed by Hebbian learning can be represented by the autonomous system
\begin{equation}\label{rnn_hl}
    \begin{aligned}
    &\dot{x}_i=-a_ix_i+\sum_{j=1}^nw_{ij}\phi(x_j), \quad i=1,...,n,\\
    &\dot{w}_{ij}=-b_{ij}w_{ij}+c_{ij}\phi(x_i)\phi(x_j), \quad i,j=1,...,n,
    \end{aligned}
\end{equation}
where $\phi(\cdot)$ is the activation function; $a_i, b_{ij}>0$ are the decaying parameters; $c_{ij}\in \mathbb{R}$ is a parameter whose absolute value can be interpreted as the learning rate.  When $c_{ij}$ takes a positive value, it is called the \emph{Hebbian learning}, and the case with $c_{ij}<0$ is the  \emph{anti-Hebbian learning} \cite{dayan2005theoretical}.  There exist many kinds of activation functions, such as the logistic sigmoid function, ReLU, and $tanh$. In this work, we consider exclusively the standard sigmoid activation function, i.e., $\phi(z):= 1/(1+e^{-z}) \in (0,1), ~z\in \mathbb{R}$. 

System \eqref{rnn_hl} consists of a feedback connection of the standard RNN model and the second-order Hebbian learning rule \cite{Gerstner2002}. We term it as the \emph{generalized RNN-HL model}, because of the expanded parameter space, i.e., the additional parameters $a_{i}$, $b_{ij}$, and $c_{ij}$ compared with those models in  \cite{kozachkov2020achieving,9993009}. This general model allows us to explore the dynamics of neurons and synapses in diverse and non-homogeneous settings. The dimension of system \eqref{rnn_hl} will vary between $2n-1$ and $n+n^2$ depending on synaptic connection structures.  The RNN-HL dynamics are in general too difficult to analyze because of the high dimension and nonlinear activation functions. So we are going to analyze the system dynamics by starting from simplified scenarios.

\subsection{Minimal models of RNN-HL}
 In neural networks, a synaptic connection involves a pair of neurons. These two neurons can have different synaptic connections between them, regarding the inhibitory and excitatory synapses. In the case of only two neurons, there exist two different minimal connection structures (as shown in Fig. \ref{fig:minimal_network}). In the first case,  the connections between Neurons $1$ and $2$ are bidirectional and can have asymmetric weights. 
\begin{figure}
\centering
	\includegraphics[width=8cm]{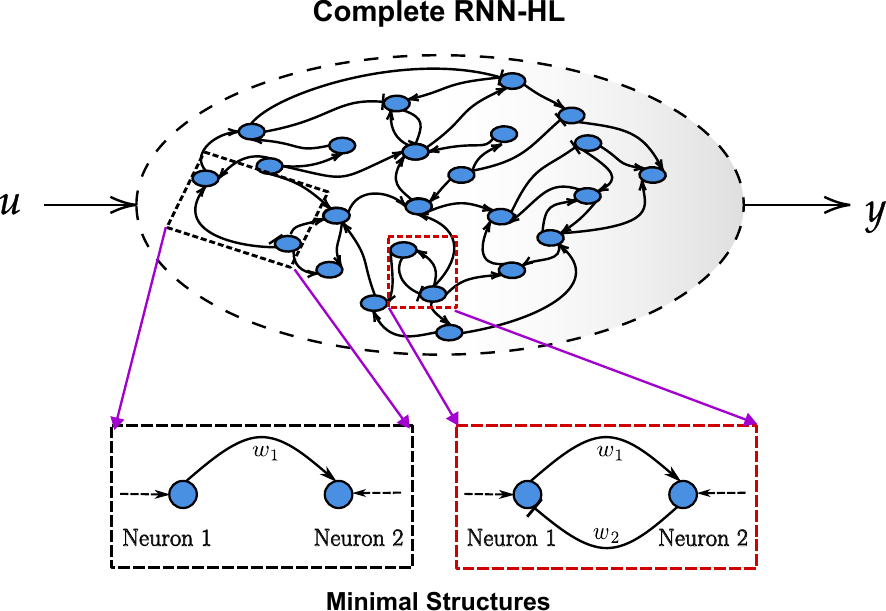}
	\caption{The minimal structures of a large RNN-HL network, where the edge with an arrow denotes an excitatory synapse, while the edge with a bar denotes an inhibitory synapse connection.}\label{fig:minimal_network}
\end{figure}
According to the RNN-HL model \eqref{rnn_hl}, the dynamics on this minimal network are given by the system
\begin{equation}\label{minimalnetwork}
 \begin{aligned}
 &\dot{x}_1=-a_1x_1+w_2\phi(x_2)+u_1\\
 &\dot{x}_2=-a_2x_2+w_1\phi(x_1)+u_2\\
 &\dot{w}_{1}=-b_{1}w_{1}+c_{1}\phi(x_1)\phi(x_2)\\
 &\dot{w}_{2}=-b_{2}w_{2}+c_{2}\phi(x_1)\phi(x_2),
 \end{aligned} 
\end{equation}
where $u_1(t), u_2(t) \in \mathbb{R}$ capture all the inputs from other nodes as well as external sources, and parameters $c_1,c_2 \in \mathbb{R}$ are non-zero. 
This minimal network model is $4$ dimensional. We are interested in the intrinsic behaviors of this model, i.e., system \eqref{minimalnetwork} without inputs. In the next section, we will investigate the dynamics by studying the equilibria and stability.

In the second case, when there is only one synapse between the two neurons, it does not matter which neuron is pre- or post-synaptic mathematically. 
So we consider the scenario with a single synapse that is directed from Neuron $1$ to Neuron $2$. Thus, the model is given by
\begin{equation}\label{minimalnetwork_onesynapse}
 \begin{aligned}
 &\dot{x}_1=-a_1x_1+u_1\\
 &\dot{x}_2=-a_2x_2+w_1\phi(x_1)+u_2\\
 &\dot{w}_{1}=-b_{1}w_{1}+c_{1}\phi(x_1)\phi(x_2).
 \end{aligned} 
\end{equation}
In the upcoming section, we will conduct a detailed analysis of the dynamics of both models. We are interested in the intrinsic behaviors of these models, i.e., system \eqref{minimalnetwork} or \eqref{minimalnetwork_onesynapse} without inputs. When the inputs remain constant, the dynamical system with inputs can be considered a regular perturbation of the autonomous system. Consequently, our emphasis will be on the autonomous versions of these two models.  We will investigate the dynamics by studying the equilibria and stability.

\section{Main Results of Minimal Models}
\subsection{Equilibria characterization}\label{equilibria_characterization}
In the absence of inputs, by letting the right-hand side of the two models \eqref{minimalnetwork} and \eqref{minimalnetwork_onesynapse} equal zero, one can compute the equilibria of the autonomous systems. 

We start with the second minimal model \eqref{minimalnetwork_onesynapse}, as it is much easier to analyze. The equilibria of this system are the solutions to the equations
\begin{equation}\label{eq_equations_onesynapse}
 \begin{aligned}
 &a_1x_1=0\\
 &a_2x_2=w_1\phi(x_1)\\
&b_{1}w_{1}=c_{1}\phi(x_1)\phi(x_2).
 \end{aligned} 
\end{equation}
The above equations can be reduced to be 
\[4a_2b_1x_2=c_1\phi(x_2).\]
As $\phi(x_2)$ is the standard sigmoid function, the equation admits a unique solution for all parameter conditions \cite{beer1995dynamics}. It follows that \eqref{minimalnetwork_onesynapse} always has a unique equilibrium.

Similarly, for the system \eqref{minimalnetwork}, the equilibria are the solutions to the following equations.
\begin{equation}\label{eq_equations}
 \begin{aligned}
 &a_1x_1=w_2\phi(x_2)\\
 &a_2x_2=w_1\phi(x_1)\\
 &b_{1}w_{1}=c_{1}\phi(x_1)\phi(x_2)\\
 &b_{2}w_{2}=c_{2}\phi(x_1)\phi(x_2).
 \end{aligned} 
\end{equation}

Considering $\phi(\cdot)\in (0,1)$, one can substitute $w_1=\frac{a_2x_2}{\phi(x_1)}$ and $w_2=\frac{a_1x_1}{\phi(x_2)}$ into the last two equations. Then,
the above equations turn out to be 
\begin{equation}\label{eq_equations1}
 \begin{aligned}
 &b_{1}a_2x_2=c_{1}\phi^2(x_1)\phi(x_2)\\
 &b_{2}a_1x_1=c_{2}\phi(x_1)\phi^2(x_2).
 \end{aligned} 
\end{equation}
However, due to the presence of logistic sigmoid functions, it is difficult or impossible to solve equation \eqref{eq_equations1} analytically. Instead of trying to get analytical solutions, we will try to determine the existence of the solutions and their positions in an alternative way.

Let $X:=[x_1,x_2,w_1,w_2]^\top$ be the state vector.
Based on \eqref{eq_equations}, we can define a function
\[F(X):=\begin{pmatrix}
    &\frac{w_2\phi(x_2)}{a_1}\\
    &\frac{w_1\phi(x_1)}{a_2}\\
    &\frac{c_1\phi(x_1)\phi(x_2)}{b_1}\\
    &\frac{c_2\phi(x_1)\phi(x_2)}{b_2}
    \end{pmatrix}
.\]
Then, proving \eqref{eq_equations} has solutions is equivalent to proving that the mapping $F$ admits fixed points, i.e., there exists a $X^*$ such that $F(X^*)=X^*$.

Next, define the set
\[\chi:=\left \{(x_1,x_2,w_1,w_2)\in \mathbb{R}^4: \begin{aligned}
&|x_1|,|x_2|\leq x_{\rm{max}}\\
&|w_1|,|w_2|\leq w_{\rm{max}}
\end{aligned} \right\},\]
with $w_{\rm{max}}:=\max\{|c_1|,|c_2|\}/\min\{b_1,b_2\}$ and $x_{\rm{max}}:=w_{\rm{max}}/\min\{a_1,a_2\}$. 
According to \cite[Lemma IV.1]{9993009}, $\chi$ is forward invariant and attractive with respect to \eqref{minimalnetwork}.  Then, by using Brouwer's fixed point theorem \cite{karamardian2014fixed}, we obtain the following result.
\begin{lemma}
System \eqref{minimalnetwork} has at least one equilibrium, and all the equilibria are contained in $\chi$.   
\end{lemma}
\begin{proof}
The proof is a straightforward application of Brouwer's fixed point theorem, which states that for any continuous function 
that maps a nonempty compact convex set to itself, it has at least one equilibrium.

From the definition of $\chi$, one can check that $\chi$ is a closed hyperbox. The defined mapping $F(X)$ is continuous. Let $\tilde{X}=:(\tilde{x}_1, \tilde{x}_2,\tilde{w}_1,\tilde{w}_2)$ be an arbitrary point in $\chi$, one can check that $F(\tilde{X}) \in \chi$. It means $F(X)$ maps $\chi$ to itself.  Therefore, the existence of fixed points for $F(X)$ in $\chi$ follows immediately, which in turn implies the existence of equilibria for \eqref{minimalnetwork}. 
It is noted that $\chi$ is forward invariant and attractive for system \eqref{minimalnetwork}. It means that every solution of \eqref{minimalnetwork} is bounded to $\chi$. Then, according to \cite[Lemma 4.1]{khalil1996nonlinear}, it follows that the positive limit sets (including the equilibria) for system \eqref{minimalnetwork} is a subset of $\chi$. 
\end{proof}
The invariance and attraction of $\chi$ helps determine the location of equilibria, as it guarantees that all the positive limit sets, such as equilibria, limit cycles, etc., are contained in it. 
The set $\chi$ itself gives an estimate of the position of all possible equilibria of the system. However, this estimate can be vague, thus we are going to study the position and number of possible equilibria in depth.
Let $(x_1^*,x_2^*,w_1^*,w_2^*) \in \mathbb{R}^4$ denote an equilibrium of system \eqref{minimalnetwork}.
Based on \eqref{eq_equations1}, we can roughly judge the located region of all the possible equilibria. It is easy to check that $x_1^*$ and $w_1^*$ have the same sign with $c_1$, while $x_2^*$ and $w_2^*$ have the same sign with $c_2$. And we have the following observations.
\begin{enumerate}
    \item When $c_1>0$ and $c_2>0$, one has $x_1^*>0$, $x_2^*>0$, which leads to $w_1^*>0$, $w_2^*>0$. Therefore, the equilibria are located in the positive quadrant $\mathbb{R}^4_+$.
    \item When $c_1<0$ and $c_2<0$, one has $x_1^*<0$, $x_2^*<0$, which leads to $w_1^*<0$, $w_2^*<0$ similarly. Therefore, the equilibria are located in the negative quadrant $\mathbb{R}^4_-$.
   \item When $c_1$ and $c_2$ are of different signs,  $x_1^*$ and $w_1^*$ have different signs with $x_2^*$ and $w_2^*$. Therefore, the equilibria will be located in the other two quadrants accordingly.
\end{enumerate}

In general, a nonlinear neural system can have one or multiple equilibria under different parameter conditions. This is also true for system \eqref{minimalnetwork}. In the following, we will discuss the conditions that characterize the different cases.

One can calculate the Jacobian of $F$ at a point $X\in \chi$, namely
\begin{equation}
\begin{aligned}
    \jac{F}(X)&= \begin{bmatrix}
    0&\frac{w_2\phi'(x_2)}{a_1}&0&\frac{\phi(x_2)}{a_1}\\
    \frac{w_1\phi'(x_1)}{a_2}&0&\frac{\phi(x_1)}{a_2}&0\\
    \frac{c_1\phi'(x_1)\phi(x_2)}{b_1}&\frac{c_1\phi(x_1)\phi'(x_2)}{b_1}&0&0\\
    \frac{c_2\phi'(x_1)\phi(x_2)}{b_2}&\frac{c_2\phi(x_1)\phi'(x_2)}{b_2}&0&0
    \end{bmatrix}.
    \end{aligned}
\end{equation}
Let $\lvert \jac{F}(X)\rvert$ be the matrix norm of the Jacobian. We consider  the $1$-norm, i.e., 
\begin{equation}\label{1normJacobian1}
\begin{aligned}
    &\lvert \jac{F}(X)\rvert=\\
    &\max \left\{ 
     \left\lvert\frac{w_2\phi'(x_2)}{a_1} \right \rvert+ \left\lvert\frac{c_1\phi(x_1)\phi'(x_2)}{b_1} \right \rvert+ \left\lvert\frac{c_2\phi(x_1)\phi'(x_2)}{b_2} \right \rvert, \right.\\ 
    &\quad\quad~~\left. \left\lvert\frac{w_1\phi'(x_1)}{a_2} \right \rvert+ \left\lvert\frac{c_1\phi'(x_1)\phi(x_2)}{b_1} \right \rvert+ \left\lvert\frac{c_2\phi'(x_1)\phi(x_2)}{b_2} \right \rvert,\right.\\ &\quad\quad~~\left. \left\lvert\frac{\phi(x_1)}{a_2} \right \rvert, \left\lvert\frac{\phi(x_2)}{a_1} \right \rvert\right\}.
    \end{aligned}
\end{equation}
We have the following result on the existence of a unique equilibrium.  
\begin{lemma}\label{uniqueness}
When the conditions
\begin{equation}\label{conditionofunqueequilibrium}
  a_1,a_2>1,~\max \left\{\frac{\lvert w\rvert_{\max} }{a_1},\frac{\lvert w\rvert_{\max} }{a_2} \right\}+\frac{|c_1|}{b_1}+\frac{|c_2|}{b_2} <4.
\end{equation}
hold, system \eqref{minimalnetwork} has a unique equilibrium in $\chi$. 
\end{lemma}
\begin{proof}
Let $\sup_{X\in \chi} \lvert \jac{F}(X)\rvert $ denote the superior of the matrix norm of $\jac{F}$ in the set $\chi$.
It is noted that $\phi(\cdot)\in (0,1)$ and $\phi'(\cdot)=e^{-x}/(1+e^{-x})^2 \in (0,1/4)$. We have for all $X\in \chi$ that
\[\sup \left\lvert\frac{\phi(x_1)}{a_2} \right \rvert=\frac{1}{a_2},\]
\[ \sup \left\lvert\frac{\phi(x_2)}{a_1} \right \rvert=\frac{1}{a_1},\]
\begin{equation*}
    \begin{aligned}
    & \sup \left(\left\lvert\frac{w_1\phi'(x_1)}{a_2} \right \rvert+ \left\lvert\frac{c_1\phi'(x_1)\phi(x_2)}{b_1} \right \rvert+ \left\lvert\frac{c_2\phi'(x_1)\phi(x_2)}{b_2} \right \rvert \right) \\
    &=\frac{\lvert w_1\rvert_{\max} }{4a_2}+\left\lvert\frac{c_1}{4b_1}\right \rvert+\left\lvert\frac{c_2}{4b_2}\right \rvert,\\
    &\sup \left(\left\lvert\frac{w_2\phi'(x_2)}{a_1} \right \rvert+ \left\lvert\frac{c_1\phi(x_1)\phi'(x_2)}{b_1} \right \rvert+ \left\lvert\frac{c_2\phi(x_1)\phi'(x_2)}{b_2} \right \rvert \right) \\
    &=\frac{\lvert w_2\rvert_{\max} }{4a_1}+\left\lvert\frac{c_1}{4b_1}\right \rvert+\left\lvert\frac{c_2}{4b_2}\right \rvert.
    \end{aligned}
\end{equation*}
 If conditions \eqref{conditionofunqueequilibrium} hold, it is easy to check that  $\sup_{X \in \chi} \lvert \jac{F}(X)\rvert <1$. It then yields that for any $X_1, X_2\in \chi$
\[\lvert F(X_1)-F(X_2)\rvert <\sup_{X\in \chi} \lvert \jac{F}(X)\rvert \lvert X_1-X_2\rvert,\]
which implies that $F(X)$ is a contraction mapping in $\chi$. Hence, it follows that there exists a unique fixed point for the map $F(X)$ in the set $\chi$. Equivalently, according to the Banach fixed-point theorem, a unique equilibrium exists for system \eqref{minimalnetwork} on $\chi$.
\end{proof}
It is noted that in the above process, we used the $1$-norm for $DF$, and
 \eqref{conditionofunqueequilibrium} is sufficient conditions for the existence of a unique equilibrium.
 One can obtain different conditions by using other norms. 
The uniqueness of the equilibrium result can also be readily extended to the case with the presence of inputs, just like that in the standard Hopfield neural network model. We refer readers to references \cite{Guan2000,1000277,Wang2010}.

We are now turning to the scenario of multiple equilibria.
Due to many parameters in system \eqref{minimalnetwork} and the sigmoid functions, it is extremely difficult to characterize the conditions for multiple equilibria. Thus, we make some simplifications and study the possibility of the existence of multiple equilibria based on the equation \eqref{eq_equations2}.

Recall that $a_1$, $a_2$, $b_1$, $b_2>0$. According to the condition in Theorem \ref{uniqueness}, $a_1$ and $a_2$ have to be not greater than $1$ to have multiple equilibria. On the other hand, we can term $b_2a_1$ and $b_1a_2$ as a single parameter respectively, as they appear together in  \eqref{eq_equations2}.
Now, we assume the equalities with respect to the parameters hold, e.g., 
\begin{equation}\label{parasymmetry}
    b_2a_1=b_1a_2=A,~~c_1=c_2=c.
\end{equation}

Under \eqref{parasymmetry}, the condition \eqref{eq_equations1} turns to be
\begin{equation}\label{eqconditon1_1}
 \begin{aligned}
 &Ax_2=c\phi^2(x_1)\phi(x_2)\\
 &Ax_1=c\phi(x_1)\phi^2(x_2).
 \end{aligned} 
\end{equation}
It is easy to recognize that the above equations remain unchanged after the interchange of $x_1$ and $x_2$. It means that there is a reflection symmetry with respect to  the $x_1=x_2$ plane. Thus, if there is a solution, one can obtain another solution by simply swapping the values of the two variables. 

Let $(x_1^*,x_2^*)$ be a solution of \eqref{eqconditon1_1}.
We consider the specific case when $x_1^*=x_2^*$. Then, \eqref{eqconditon1_1} yields to
\[Ax_1=c\phi^3(x_1).\]
It is not difficult to check that the above equation has a unique root for all $c$, which is denoted by $\hat{x}$. Hence, there always exists a solution $(\hat{x},\hat{x})$ to \eqref{eqconditon1_1}.

Let $\alpha:=\frac{b_1a_2x_2}{c_1\phi(x_2)}$. By unfolding $x_1$ in \eqref{eqconditon1_1}, one can obtain 
\begin{equation}\label{eq_equations2}
A(\ln{\sqrt{\alpha}}-\ln{(1-\sqrt{\alpha})})=c\sqrt{\alpha}\phi^2(x_2).
\end{equation}

Based on \eqref{eq_equations2}, we can define the following function
\[\begin{aligned}
f(\xi)&:=A(\ln{\sqrt{\alpha}}-\ln{(1-\sqrt{\alpha})})-c\sqrt{\alpha}\phi^2(\xi)\\
&=\ln{\frac{\sqrt{\xi(1+e^{-\xi})/c}}{(1-\sqrt{\xi(1+e^{-\xi})/c})}}\\
&\quad -\frac{c\sqrt{\xi(1+e^{-\xi})/c}}{(1+e^{-\xi})^2}, ~\xi\in (\underline{\xi}, \overline{\xi}) \subset \mathbb{R}.
\end{aligned}\]
$f(\xi)$ is continuous and differentiable in the defined interval $\xi\in (\underline{\xi},\overline{\xi})$ which will be specified later. 
Without loss of generality, we consider $A=1$ and examine $f(\xi)$ in two cases, i.e., $c>0$ and  $c<0$. \\
\noindent
\textbf{Case I:} If $c>0$, one has $\xi\in (0,\beta_c)$ where $\beta_c$ is the root of the equation $\xi(1+e^{-\xi})=c$. And, it is easy to check that $f(\xi)\rightarrow -\infty$ as $\xi\rightarrow 0$, while $f(\xi)\rightarrow +\infty$ as $\xi\rightarrow \beta_c$. This implies that there is some $\xi^*$ such that $f(\xi^*)=0$. One can check that there only exists one $\xi^*$.
Correspondingly, \eqref{eq_equations1} has only one solution $(\xi^*,\xi^*)$ and the system \eqref{minimalnetwork} has a unique equilibrium $(\xi^*,\xi^*, \frac{a_2\xi^*}{\phi(\xi^*)}, \frac{a_1\xi^*}{\phi(\xi^*)})$. 

\noindent
\textbf{Case II:} If $c<0$, one has $\xi\in (-\beta_c,0)$. Similarly, one can check that $f(\xi)\rightarrow -\infty$ as $x_2\rightarrow 0$, while $f(\xi)\rightarrow +\infty$ as $\xi\rightarrow -\beta_c$, which implies there is at least one $\xi^*$ satisfying  $f(\xi)=0$. One can also check that only one $\xi^*$ exists when $-123.7215 \leq c<0$, while three such $\xi^*$ exist when $c<-123.7215$.
In the latter situation, \eqref{eq_equations1} has three solutions, i.e., $(\xi_1^*,\xi_1^*)$, $(\xi_2^*,\xi_3^*)$, and $(\xi_3^*,\xi_2^*)$, and the system \eqref{minimalnetwork} has three equilibria accordingly. 

We have shown that the system \eqref{minimalnetwork} has an odd number of equilibria, e.g., $1$ or $3$, in a specific case with the parametric symmetry. It is worth noting that the system can also have $2$ equilibria in the situation when the symmetry is broken as shown in Fig. \ref{fig:equilibria}.
\begin{figure}[htbp!]
    \centering
    \includegraphics[width=8.5cm]{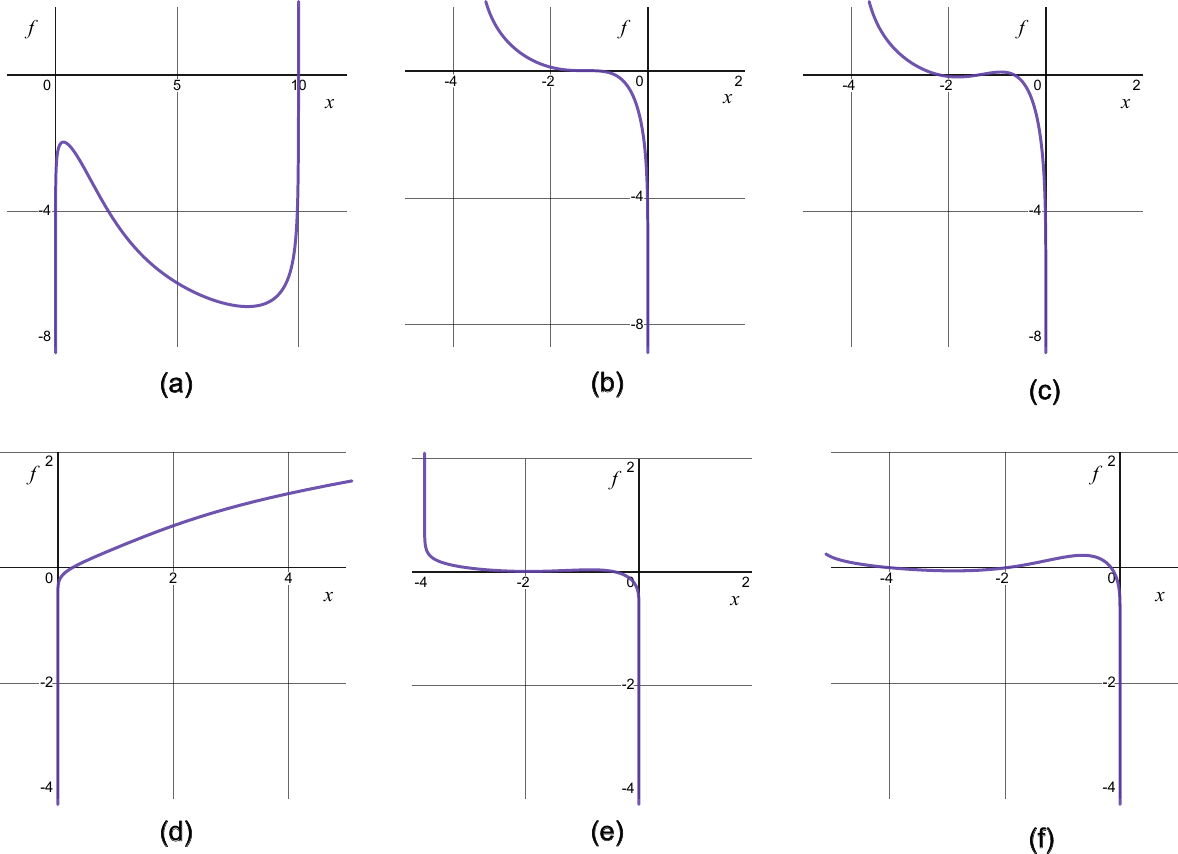}
    \caption{The solutions to $f(\xi)=0$ in different cases. }
    \label{fig:equilibria}
\end{figure}

In the above analysis, it has been noticed there is a critical value of $c\approx -123.7215$ such that the number of equilibria changes at that point, which typically implies a bifurcation occurring therein. We will investigate this emerging phenomenon formally in the following section. 

\subsection{Bifurcation Analysis}

Again, we start from system \eqref{minimalnetwork_onesynapse}. Denote the unique equilibrium by $(0,\tilde{x}_2, 2\tilde{x}_2)$, where $\tilde{x}_2$ is the solution to $4a_2b_1x_2=c_1\phi(x_2)$ and $2\tilde{x}_2$ comes from substituting $\tilde{x}_2$ to obtain $w_1$.
One can calculate the Jacobian of \eqref{minimalnetwork_onesynapse} at $(0,\tilde{x}_2, 2\tilde{x}_2)$, i.e.,
\begin{equation}\label{jacobian_onesynapse}
  \tilde{J}=\begin{bmatrix}
       -a_1&0&0\\
       \tilde{x}_2/2&-a_2&1/2\\
c_1\tilde{\phi}_2/4&c_1\tilde{\phi}_2'/2&-b_1
   \end{bmatrix},
\end{equation}
where $\tilde{\phi}_2=\phi(\tilde{x}_2)$ and $\tilde{\phi}_2'$ is the derivative at that point.
It is easy to check that
\[\det(\tilde{J})=\frac{a_1c_1\tilde{\phi}_2'}{4} - a_1a_2b_1=-a_1a_2b_1[1-\tilde{x}_2(1-\tilde{\phi}_2)]<0,\]
where we have used $\tilde{\phi}_2'=\tilde{\phi}_2(1-\tilde{\phi}_2)$ and $4a_2b_1\tilde{x_2}=c_1\phi(\tilde{x_2})$. The always negative determinant implies that the equilibrium does not change stability as the parameter changes. And from the eigenvalues, we further know that the equilibrium is always exponentially stable because 
\begin{equation}
    \begin{aligned}
        &\lambda_1=-a_1<0,\\
&\lambda_{2,3}=-\frac{a_2+b_1}{2} \pm \frac{((a_2-b_1)^2+c_1\tilde{\phi}_2')^{1/2}}{2}<0. 
\end{aligned}
\end{equation}
Thus, we have the following statement.
\begin{theorem}
    The unique equilibrium $(0,\tilde{x}_2, 2\tilde{x}_2)$ of system \eqref{minimalnetwork_onesynapse} is always exponentially stable.
\end{theorem}

Next, we turn to the system \eqref{minimalnetwork}. We will focus on the specific parameter condition where 
\begin{equation}\label{specificcase_para}
  a_1=a_2=b_1=b_2=1,~~c_1=c_2=c\in \mathbb{R}.  
\end{equation}
The following results are rigorously proved for system \eqref{minimalnetwork} under this specific parameter condition. However, for other parameter conditions, similar results are also observed as shown numerically later.

Under \eqref{specificcase_para}, in the absence of inputs, system \eqref{minimalnetwork} becomes
\begin{equation*}
 \begin{aligned}
 &\dot{x}_1=-x_1+w_2\phi(x_2)\\
 &\dot{x}_2=-x_2+w_1\phi(x_1)\\
 &\dot{w}_{1}=-w_{1}+c\phi(x_1)\phi(x_2)\\
 &\dot{w}_{2}=-w_{2}+c\phi(x_1)\phi(x_2).
 \end{aligned} 
\end{equation*}
It is noticed that the vector fields of $\dot{w}_1$and $\dot{w}_2$ are the same, which implies that $w_1(t)=w_2(t)$ as $t\rightarrow \infty$. Thus, we can reduce the system to the  3-dimensional subsystem
\begin{equation}\label{minimalnetwork1}
 \begin{aligned}
 &\dot{x}_1=-x_1+w\phi(x_2)\\
 &\dot{x}_2=-x_2+w\phi(x_1)\\
 &\dot{w}=-w+c\phi(x_1)\phi(x_2).
 \end{aligned} 
\end{equation}
Denote the vector field of \eqref{minimalnetwork1} by 
    \[G(x_1,x_2,w)=\begin{pmatrix}
      &-x_1+w\phi(x_2)\\
      &-x_2+w\phi(x_1)\\
      &-w+c\phi(x_1)\phi(x_2)
    \end{pmatrix}.\]
Define a map  $S: \mathbb{R}^3\rightarrow \mathbb{R}^3$ by $ S(x_1,x_2,w)=(x_2,x_1,w)$.
A simple computation shows that 
\[\dot{S}(x_1,x_2,w)=(G\circ S)(x_1,x_2,w),\]
where the symbol $\circ$ is used to denote function composition.
This equability implies that $S(x_1,x_2,w)$ is a $\mathbb{Z}_2$ symmetry \cite{kuznetsov1998elements} in \eqref{minimalnetwork1}. 

Then, system \eqref{minimalnetwork1} can be called $\mathbb{Z}_2$-equivariant with respect to $S$. This property has important consequences on the dynamics of \eqref{minimalnetwork1}.
If $(x_1^*, x_2^*, w^*)$ is an equilibrium  of the system, then $(x_2^*, x_1^*, w^*)$ is also an
equilibrium  of the system. This is consistent with the obtained result in Section III.A.
Moreover, the two equilibria have the same type of stability \cite{kuznetsov1998elements}. 
We have shown that there is always an equilibrium in \eqref{minimalnetwork1}, which is located on the plane $L:= \{(x_1,x_2,w)\in \mathbb{R}^3: 
x_1=x_2$\}. This plane is invariant under the flow of the system, and it divides the whole space $\mathbb{R}^3$ into two subsets such that trajectories originating from a specific subset will remain entirely within that subset.

We denote the equilibrium on the plane $L$ by $\hat{X}:=(\hat{x},\hat{x},c\phi(\hat{x})^2)$, where $\hat{x}$ is the unique solution to $x=c\phi(x)^3, c\in \mathbb{R}$.  This equilibrium is fixed as it remains the same under the symmetry.
Let $\hat{J}$ be the Jacobian matrix of system \eqref{minimalnetwork1} at  $\hat{X}$. We have that 
\begin{equation}\label{jacobian}
   \hat{J}=\begin{bmatrix}
       -1&c\hat{\phi}^2\hat{\phi}'&\hat{\phi}\\
       c\hat{\phi}^2\hat{\phi}'&-1&\phi\\
c\hat{\phi}'\hat{\phi}&c\hat{\phi}'\hat{\phi}&-1
   \end{bmatrix}, 
\end{equation}
where $\hat{\phi}=\phi(\hat{x})$ and $\hat{\phi}'$ is the derivative at that point. We know that for the sigmoid function, it admits $\hat{\phi}'=\hat{\phi}(1-\hat{\phi})$. Substituting this into $\hat{J}$ yields 
\begin{equation*}
   \hat{J}=\begin{bmatrix}
       -1&c\hat{\phi}^3(1-\hat{\phi})&\hat{\phi}\\
       c\hat{\phi}^3(1-\hat{\phi})&-1&\hat{\phi}\\
c\hat{\phi}^2(1-\hat{\phi})&c\hat{\phi}^2(1-\hat{\phi})&-1
   \end{bmatrix}.
\end{equation*}
The characteristic polynomial of this matrix is 
\[\det(\hat{J}-\lambda I)=\lambda^3+3\lambda^2+B\lambda-\det(\hat{J}),\]
where 
\[B=-\hat{\phi}^8 + 2c^2\hat{\phi}^7 - c^2\hat{\phi}^6 + 2c\hat{\phi}^4 - 2c\hat{\phi}^3 + 3,\]
and $\det(\hat{J})$ is the determinant of $\hat{J}$, i.e.,
\[\det(\hat{J})=3(c\hat{\phi}^4 - c\hat{\phi}^3 +1)( c\hat{\phi}^4 - c\hat{\phi}^3  - 1)
.\]
It is noticed that the above polynomial can be factorized as
\[(\lambda-c\hat{\phi}^4 + c\hat{\phi}^3 + 1)[\lambda^2+(2+c\hat{\phi}^4-c\hat{\phi}^3)\lambda-3(c\hat{\phi}^4-c\hat{\phi}^3+1)].\]
Then, we obtain the eigenvalues of $\hat{J}$, i.e.,
\begin{equation}
    \begin{aligned}
        &\lambda_1=c\hat{\phi}^4 - c\hat{\phi}^3 - 1,\\
&\lambda_{2,3}=\frac{c\hat{\phi}^3 - c\hat{\phi}^4 - 2}{2}\pm \frac{((c\hat{\phi}^3-c\hat{\phi}^4)(- c\hat{\phi}^4 + c\hat{\phi}^3 + 8))^{1/2}}{2}. 
\end{aligned}
\end{equation}
By analyzing the eigenvalues, we have the following result. 
\begin{lemma}\label{eigenvalue_stability}
 There exists a critical value of $c_0<0$ such that the equilibrium $\hat{X}$ is asymptotically stable when $c>c_0$ and unstable when $c<c_0$.   
\end{lemma}
\begin{proof}
    The proof relies on carefully examining the eigenvalues of the Jacobian. See Appendix \ref{proofoflemma_stability} for details.
\end{proof}

From Section \ref{equilibria_characterization},
we have shown that the number of equilibria of system \eqref{minimalnetwork1} changes, and the stability of the equilibrium $\hat{X}$ also changes at $c=c_0\approx -123.7215$ as stated above. These two values are consistent. Now, it is ready to classify the type of the bifurcation formally.
\begin{theorem}
    System \eqref{minimalnetwork1} undergoes a pitchfork bifurcation with respect to the equilibrium $\hat{X}$  at $c=c_0$.
\end{theorem}
\begin{proof}
   It is already known that the system \eqref{minimalnetwork1} is $S$-symmetric. The equilibrium $\hat{X}$ has a simple eigenvalue $\lambda_1=0$ at $c=c_0$. And one can calculate the corresponding eigenvector is $v_1=[-1,1,0]^\top$. Thus, this eigenvector does not belong to the subspace defined by the plane $x_1=x_2$. 

   In this case, the conditions of \cite[Theorem 7.7]{kuznetsov1998elements} (p. 281) are satisfied, and it results in that the system \eqref{minimalnetwork1} has a one-dimensional $S$-invariant center manifold for all $c$ in the vicinity of $c=c_0$. According to \cite{kuznetsov1998elements}, the restriction of the system to this center manifold is locally
topologically equivalent near the equilibrium to the following normal form 
\[\dot{y}=\beta y\pm y^3, ~~y\in \mathbb{R}.\]
The genericity conditions are satisfied trivially. Therefore, one can conclude that a pitchfork bifurcation of $\hat{X}$ happens at $c=c_0$.
\end{proof}
The identified pitchfork bifurcation is \emph{supercritical} as it results in the appearance of two $S$-conjugate equilibria for $c<c_0$, which are both stable, as shown in Fig. \ref{fig:exaplesimulation}D. 
With the result of pitchfork bifurcation,  we can describe the local dynamics of system \eqref{minimalnetwork1}: when $c<c_0$, the system has three equilibria, among which two are stable and one is unstable; the unstable equilibrium becomes stable and the two stable equilibria disappear as $c$  passes $c_0$. 

The results obtained in this section are proved for the specific symmetric case \eqref{specificcase_para}. However, we note that all these findings can also exist in more general cases. The pitchfork bifurcation will remain for other parameter conditions when the symmetry in \eqref{specificcase_para} is preserved. Moreover, breaking the symmetry can result in the unfolding of pitchfork bifurcation, so-called the \emph{imperfect pitchfork bifurcation} (as shown in Fig. \ref{fig:exaplesimulation}(b)).  The extensions to encompass broader cases, as well as the results for non-symmetric cases, will depend on the sophisticated application of singular bifurcation theory \cite{golubitsky2012singularities}, and we leave them to future study.

\begin{figure*}[htbp!]
    \centering
    \includegraphics[width=14.5cm]{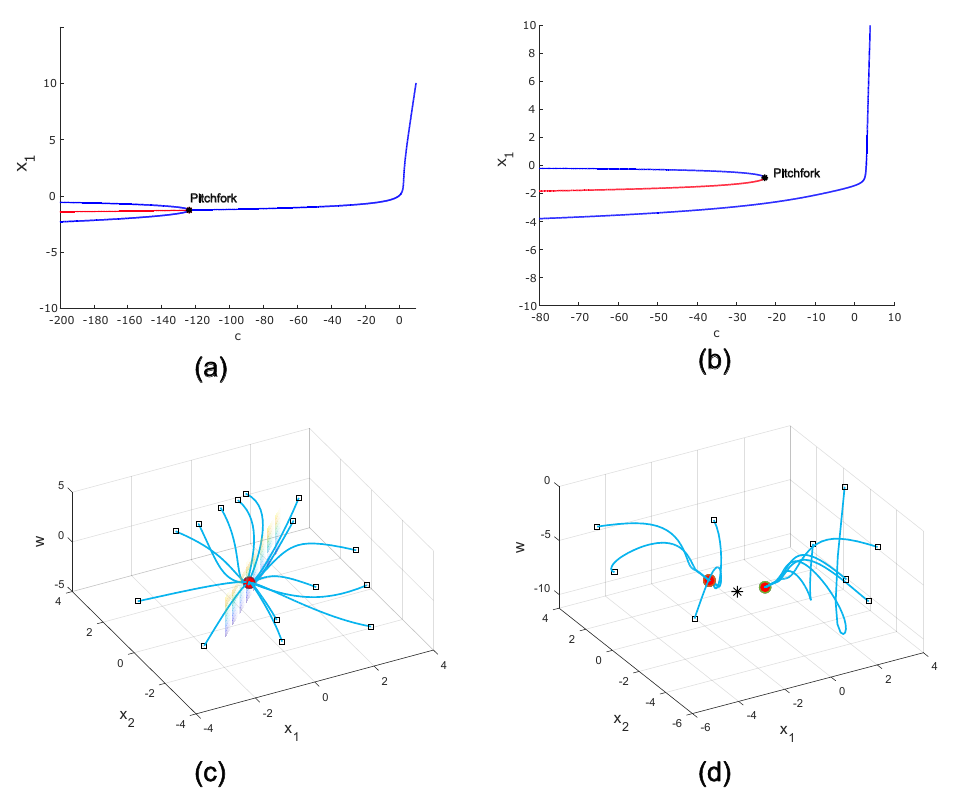}
    \caption{Bifurcations and system dynamics. (a) shows the pitchfork bifurcation in the symmetric parameters case, where $a_1=1$, $a_2=1$, $b_1=1$.  (b) shows the pitchfork bifurcation in the asymmetric parameters case, where $a_1=0.2$, $a_2=0.4$, $b_1=0.25$ and $b_2=0.5$. The blue lines represent the stable equilibria and the red lines are the unstable equilibria. 
    (c) depicts the dynamics of \eqref{minimalnetwork1} with a uniquely stable equilibrium (red dot) when $c=-3$, while (d) depicts the dynamics \eqref{minimalnetwork1} with three equilibria when $c=-150$, of which two are stable (red dots) and one is unstable (black asterisk). The blue curves are trajectories starting from random initial points (squares).  }
    \label{fig:exaplesimulation}
\end{figure*}

Although there is only one equilibrium when $c>c_0$, whether it is globally stable or not in \eqref{minimalnetwork1} remains unknown. We are now going to prove it is indeed globally stable for some certain $c$. 
We have already known that the set $\chi$ is forward invariant and attractive for the general system \eqref{minimalnetwork}. Under the condition \eqref{specificcase_para}, this set turns to be $\chi_1:= \{(x_1,x_2,w)\in \mathbb{R}^3: 
|x_1|,|x_2|,|w|\leq |c|$\}. As all trajectories are ultimately bounded in $\chi_1$, it suffices to consider the dynamics \eqref{minimalnetwork1} restricted in $\chi_1$.

\begin{theorem}\label{theorem_globalstable}
The following statements hold for the dynamics \eqref{minimalnetwork1} in $\chi_1$:
\begin{enumerate}
    \item The plane $L$ is globally asymptotically stable for each $c>-16$;
\item The equilibrium $\hat{X}$ is globally  asymptotically stable for $c>-16$.
\end{enumerate}

\end{theorem}
\begin{proof}
The proof is detailed in Appendix \ref{proofoftheorem_globalstability}. The proof procedure is: we first prove the plan $L$ is globally asymptotically stable using the Lyapunov approach, then we study the dynamics on the plane to establish the stability of $\hat{X}$.
\end{proof}
 The above proof is indirect as we studied the reduced dynamics on the plane. This approach makes the condition obtained for $c$, i.e., $c\geq -16$ larger than that from the bifurcation theory, ie., $c\geq -123.7215$. On the other hand, we note that when $c=1$, the system is reduced to the original model of the Hopfield network model with adapting synapses as in \cite{dong1992dynamic}. In that case, there is a bounded Lyapunov-type energy function governing the system evolution, such that the dynamics are globally stable when there is a unique equilibrium. In this sense, Theorem \ref{theorem_globalstable} extends the previous result greatly in the low dimensional setting, which is not possible to obtain by directly using the method as in \cite{dong1992dynamic, Zucker2012}.

\section{Implementation in moderate-sized networks.}
We now examine the  RNN-HL model on some moderate-sized networks. We already know that the second minimal structure can exhibit the bifurcation. Thus, to conserve this structure within the whole network, we consider two specific scenarios (see Fig. \ref{fig:large_network} for illustration).   In the first scenario, we consider a network of $12$ nodes where the connections are both unidirectional and bidirectional. We vary the learning parameters for all the bidirectional connections uniformly while leaving parameters for unidirectional weights randomly initialized and fixed. For the second scenario, we consider a structured network consisting of two interconnected subnetworks.   Each subnetwork can have some neurons ($3$ and $5$ neurons have been considered), and their connections are complete. In addition, there are two interconnections associated with one neuron from each subnetwork. In this case, we only vary the learning parameters for the interconnected weights and leave the parameters of connections within each subnetwork being randomized and fixed. The network in the latter scenario is also called a network of RNNs \cite{gao2022introduction,kozachkov2022rnns}. In both cases, the minimal structure is embedded, and by investigating the number of equilibria under the varying learning parameters, we can see its impact on the network storage capacity at different levels.

For the random network setting, the number of equilibria of the RNN-HL model can be multiple or single for different ranges of the learning rate parameter value (see Fig. \ref{fig:large_network_bifurcation_full}A).  Moreover, the learning rate parameter greatly affects the number of equilibria such that it can increase the number manyfold when it decreases. 
For the interconnected network setting, as observed, the $3+3$ RNN-HL system has a similar bifurcation diagram as shown theoretically in the minimal structure (see Fig. \ref{fig:large_network_bifurcation_full}B). In the  $5+5$ RNN-HL system, a similar change of equilibria to the random case happens along the variation of the learning parameter (see Fig. \ref{fig:large_network_bifurcation_full}C). However, the changing manner is even more complex because of the complicated shape of the equilibrium curve. To sum up, the results from interconnected networks indicate that only decreasing a few parameters (the learning rates associated with interconnected weights), is enough to enable the network to exhibit more equilibria. 

\begin{figure}[htbp!]
\centering
\includegraphics[width=7cm]{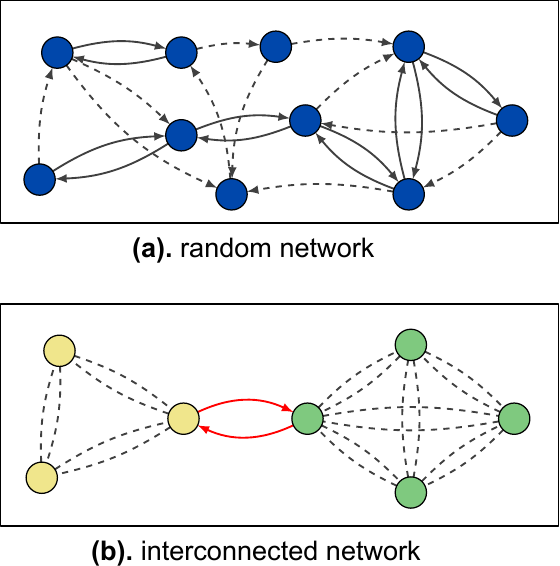}
	\caption{Schematic illustration of moderate-sized RNN-HL networks. (a). the random network; (b). the interconnected network.}\label{fig:large_network}
\end{figure}

\begin{figure*}[htbp!]
\centering
\includegraphics[width=16cm]{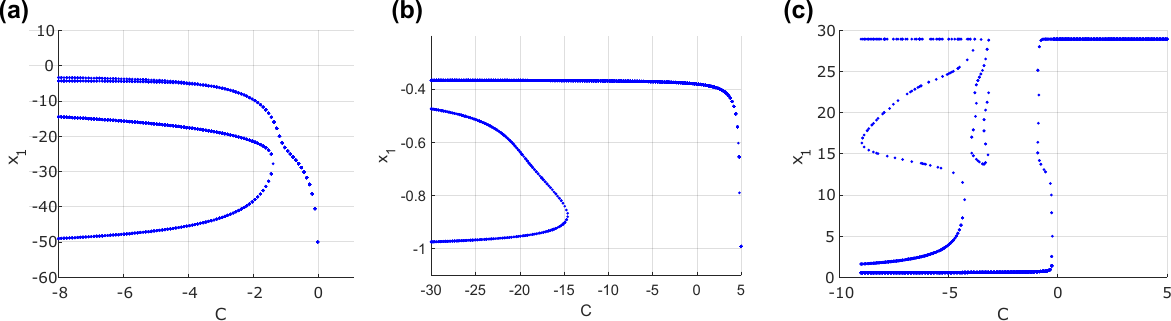}
\caption{RNN-HL network equilibria change as the learning rate varies:  (a). a random network; (b). a $3+3$ interconnected network; (c). a $5+5$ interconnected network. }
 \label{fig:large_network_bifurcation_full}
\end{figure*}

\section{Conclusions}
In this paper, we conducted a dynamical systems analysis of recurrent neural networks with different types of pairwise Hebbian learning rules.  We found that the attractor landscapes of these networks become fundamentally altered as a function of the sign and magnitude of the Hebbian learning rate. Anti-Hebbian mechanisms result in multiple equilibria via a pitchfork bifurcation, while, in contrast, Hebbian mechanisms result in a unique equilibrium and globally stable dynamics. Importantly, empirical results from simulations of moderate-sized networks corroborated these discoveries across various scenarios.
These findings suggest a differential roles for anti-Hebbian
vs. Hebbian mechanisms in shaping the attractor landscape of networks in the context of optimization and learning.

The immediate limitation of this work is that the results are only theoretically valid for minimal models, and our findings are empirically instantiated for larger-size and structured networks. 
While formal analysis of larger networks is more challenging, several tools may be useful to consider in future work, such as the
singularity theory and the equivariant bifurcation theory \cite{golubitsky2012singularities}.
Another important step will be to connect this theory to the analysis and design of recurrent neural networks and learning rules in practical artificial intelligence tasks, especially ones such as meta-learning that involve adaptation of plasticity rules \cite{tyulmankov2022meta}.

\section*{Appendix}
\subsection{Proof of Lemma \ref{eigenvalue_stability}}\label{proofoflemma_stability}
\begin{proof}
First,   we examine  $\lambda_1$. Because of $0<\hat{\phi}<1$, one has
    $\lambda_1=c\hat{\phi}^3( \hat{\phi}- 1) - 1<0$ for all $c>0$. If $c<0$, $\lambda_1$ will be zero at $c=1/(\hat{\phi}^3( \hat{\phi}- 1))$.
   As $\hat{\phi}$ is also dependent on $c$, we need to show the existence of the solution to this equation. Instead of solving $c$ directly, we substitute $\hat{x}=c\hat{\phi}^3$ into the equation, which yields
   \[\hat{x}(\hat{\phi}-1)=\frac{-\hat{x}e^{-\hat{x}}}{1+e^{-\hat{x}}}=1.\]
   The above equation can be solved uniquely for $\hat{x}$, i.e.,
   \[\hat{x}_0=-W_0(\frac{1}{e})-1,\]
   where $W_0$ is the principal branch of the \emph{Lambert W function}.
   Substituting $\hat{x}_0$ back into $\hat{x}=c\hat{\phi}^3$ gives rise the critical value of $c$, namely
   \[c_0=\hat{x}_0(1+e^{-\hat{x}_0})^3\approx -123.7215,\]
    which is consistent with the previously obtained numeric value.
As $\hat{x}_0$ is the unique value such that $\lambda_1=0$, it is obvious that $\lambda_1<0$ for all $c>c_0$. It remains to investigate the situation when $c<c_0$.
Consider 
$\lambda_1=\hat{x}(\hat{\phi}-1)-1=\frac{-\hat{x}e^{-\hat{x}}-1-e^{-\hat{x}}}{1+e^{-\hat{x}}}$.
By checking its derivative  
$(-e^{-\hat{x}}( 1- \hat{x}) + 1)/(e^{-\hat{x}} + 1)^2$, we know that the numerator is in the range $(0,-1)$ for $\hat{x}<0$. Thus, the derivative is always negative, which means that $\lambda_1$ is strictly decreasing for $\hat{x}<0$. We then can conclude that $\lambda_1>0$ for all $c<c_0$. 

Next, we turn to the eigenvalues $\lambda_{2,3}$. Note that $\lambda_{2,3}=\frac{-\lambda_1-3\pm \sqrt{ (\lambda_1+1)(\lambda_1-7)}}{2}$. We already know that $\lambda_1<0$ if $c>0$. And it is easy to check that $\lambda_1\rightarrow -1$ from right when $c\rightarrow +\infty$. Obviously, $\lambda_{2,3}$ will be a pair of complex conjugates for $\lambda_1\in (-1,7)$, and the real part $-\lambda_1-3$ is always negative in this range. For $\lambda_1>7$, $\lambda_{2,3}$ will be always negative. Therefore, we can conclude that $\lambda_{2,3}$ are negative for all $c\in \mathbb{R}$. With the signs of all the eigenvalues having been clarified,  the stability of equilibrium $\hat{X}$ is ready to judge.
\end{proof}

\subsection{Proof of Theorem \ref{theorem_globalstable}}\label{proofoftheorem_globalstability}
\begin{proof}
For statement (i), we consider
the Lyapunov function $V(x_1,x_2,w)=\frac{(x_1-x_2)^2}{2}$. It follows that 
 \[\begin{aligned}
 \dot{V}&=-(x_1-x_2)[x_1-x_2-w(\phi(x_1)-\phi(x_2))]\\
 &=-(x_1-x_2)^2\left(1+w\frac{\phi(x_1)-\phi(x_2)}{x_1-x_2}\right).
 \end{aligned}\]
It is noted that $0<\frac{\phi(x_1)-\phi(x_2)}{x_1-x_2}\leq 1/4$. For $w$, we can consider it in two cases: 1). when $c<0$, from the system equations we know that  $w(t)$ is eventually negative, which implies that $x_1$ and $x_2$ are eventually negative. It follows that $\phi(x_1), \phi(x_2)<1/2$. Then, $w(t)$ will be bounded below by $c/4$. ; 2). when $c\geq 0$, in contrast $w(t)$, $x_1$ and $x_2$ are eventually non-negative. In the latter case, it is easy to see that $\dot{V}<0$.  In case 1), it yields $c/4 \leq w\frac{\phi(x_1)-\phi(x_2)}{x_1-x_2}$ for all the points in $\chi_1$. Then, when $c>-16$, it results in $\dot{V}<0$ for all points that are not in the set $L$ and $\dot{V}=0$ for $x_1=x_2$. Thus, by LaSalle's invariance principle \cite{khalil1996nonlinear}, the plane $L$ is globally asymptotically stable. 

With the result (i), it suffices to study the reduced dynamics 
\begin{equation*}
 \begin{aligned}
 &\dot{x}_1=-x_1+w\phi(x_1)\\
 &\dot{w}=-w+c\phi(x_1)^2,
 \end{aligned} 
\end{equation*}
where $(x_1,w)\in \chi_{0}:=\{|x_1|,|w|\leq |c|\}$ and $c>-4$. For this planar system, it is easy to know that the unique equilibrium $\hat{X}_0=(\hat{x},c\phi(\hat{x})^2)$ is locally asymptotically stable by checking the eigenvalues. One can also check that the vector field points inwards at the boundary of $\chi_0$.  Moreover, the divergence of the vector field is always $-2$. The fixed sign of the divergence implies the non-existence of periodic orbits in $\chi_0$ according to the Bendixson-Dulac criterion \cite{wiggins2006introduction}. Therefore, with the existence of only one stable equilibrium and the non-existence of periodic orbits, we can conclude that the equilibrium $\hat{X}_0$ is globally asymptotically stable in the planar system for the whole set $\chi_0$ by using the Poincaré–Bendixson 
theorem \cite{wiggins2006introduction}. It further implies that $\hat{X}$ is globally asymptotically stable for $\chi_1$ in the original system, which completes the proof.
\end{proof}

\bibliographystyle{unsrt}
\bibliography{reference.bib}

\end{document}